%% file: main.tex
\pdfoutput=1
\documentclass[11pt,a4paper]{article}
\usepackage[T1]{fontenc}
\usepackage[utf8]{inputenc}
\usepackage{lmodern}
\usepackage[a4paper,margin=1in]{geometry}
\usepackage{microtype,parskip}
\usepackage{amsmath,amssymb,amsthm}
\usepackage{graphicx,booktabs,tabularx,array}
\usepackage[font=small,labelfont=bf]{caption}
\usepackage[round,authoryear]{natbib}
\usepackage{xcolor}
\usepackage{tikz}
\usetikzlibrary{arrows.meta,positioning}
\usepackage[colorlinks=true,linkcolor=black,citecolor=black,urlcolor=blue!45!black]{hyperref}
\hypersetup{pdftitle={The Last Human Gate: Forward Deployed Engineering for Governance Automation},pdfauthor={Jeremy Canale},pdfsubject={Task substitution, governance contracts, residual human work, and DGF-Bench}}
\newcolumntype{L}{>{\raggedright\arraybackslash}X}
\newcolumntype{P}[1]{>{\raggedright\arraybackslash}p{#1}}
\newtheorem{proposition}{Proposition}
\newtheorem{definition}{Definition}

\title{\textbf{The Last Human Gate}\\[0.4em]
\large Forward Deployed Engineering for Governance Automation}
\author{Jeremy Canale\\\normalsize\texttt{contact@jeremycanale.com}\\
\normalsize\url{https://www.jeremycanale.com}}
\date{September 2026}
\begin{document}
\maketitle
\begin{abstract}
Enterprise governance requires decisions, evidence, and accountable authority; it does not
require every review task to retain its current human implementation. We develop a task-substitution
framework for Digital Governance Frameworks (DGF), treating each gate as an executable contract.
Substitution requires sufficient accessible information, valid decision and authority checks, and
a reduction in total human work after exceptions, verification, correction, and maintenance are
counted. We derive a residual-work threshold and show why automating most cases can still
increase labor. Forward deployed engineering connects these conditions to an architecture for
agents, rule engines, evidence services, and escalation. DGF-Bench supplies controlled evidence
from 300 synthetic projects and 899 evaluable model--project runs. Gemini 3.8 Flash, GPT-5.6 Luna,
and DeepSeek v4.1 Flash achieve strict gate success of 94.98\%, 83.29\%, and 74.18\%; complete-route
success is 76.92\%, 42.33\%, and 24.67\%. A deterministic control passes all 1,700 gates given
the supplied rules and structured facts, locating the comparison in execution of a supplied
decision kernel. Evidence audits and 135 repeated runs distinguish
correct decisions from reliable execution. A document counterexample establishes an
information-sufficiency obstruction. These results support the technical feasibility of
replacing human execution of specified governance-review tasks with agents and software.
The framework specifies a workforce test based on the complete human effort required at fixed
output and quality; the present measurements concern review performance. Sources, dossiers,
traces, and analyses are public.
\end{abstract}
\noindent\textbf{Keywords:} task substitution; forward deployed engineering; AI agents; enterprise governance; human work; DGF-Bench.

\input{sections/01_argument}
\input{sections/01_related_work}
\input{sections/02_contract}
\input{sections/05_labor}
\input{sections/04_deployment}
\input{sections/02_experiment}
\input{sections/06_validation}
\input{sections/07_conclusion}
\clearpage
\input{sections/08_reproduction}
\clearpage
{\small\input{sections/references}}
\end{document}

%% file: sections/01_argument.tex
\section{Introduction: replacing execution, preserving the function}

Before an enterprise buys software, connects systems, or releases an application, people review
its architecture, security, contracts, cost, and operational readiness. The required outputs are
decisions and their supporting work: identify a missing agreement, reject an unsafe network
design, require a recovery test, or authorize a conditional release. These functions persist
even if software performs work previously assigned to analysts, architects, and committee
support staff. A professional title is an organizational allocation of tasks, not a technical
specification of how those tasks must be executed.

This paper argues that a Digital Governance Framework (DGF), understood as an enterprise's
connected review gates, is a tractable candidate for early substitution of human review tasks
by agentic systems. The claim concerns execution: an agent that produces a draft for someone
to redo has supplied assistance; a system that completes an accepted review has performed that
task. A workflow can contain substituted tasks and remaining human decisions at the same time.
The relevant system can combine language models, deterministic software, and escalation.
Replacing every rule check with a language-model call is neither necessary nor desirable.

Figure~\ref{fig:dgf-routes} makes this organization concrete: the same project dossier moves
through several specialist reviews, each producing a decision and its supporting record.
Buying a solution, connecting systems, and building an application require different review
sequences. General consolidates the preceding reviews into the route's governance record.

\input{figures/dgf_routes}

The DGF organizes work into dossiers, policies, named authorities, bounded decisions, and
recorded handoffs. These interfaces give engineering a concrete starting point. Operational
substitution requires accessible facts, manageable exception work, faithful evidence, and
authorized commitments. Each requirement identifies an implementation task and an acceptance
criterion. The current human allocation can change as a software implementation meets those
criteria. The proposed priority of DGF relative to other business functions is a comparative
adoption hypothesis to be evaluated across deployment settings.

We connect three contributions. First, a gate-contract formulation identifies information,
validity, and authority requirements for delegating review execution. A source counterexample
shows when no reviewer can satisfy a fixed decision contract from the permitted information.
Second, a complete labor account gives a substitution threshold: reducing routine execution
only reduces required human work if exception handling, verification, correction, and support
remain below a computable budget. Third, DGF-Bench provides a public, controlled measurement
of specified review tasks and their failure modes, including a deterministic control and
repeated trajectories. The purpose of linking them is to make the route from an agent score
to a workforce claim inspectable; none of the three substitutes for the others.

\subsection{Research questions and units of analysis}

We organize the argument around three questions with different observational units.
\emph{First, can software execute a specified gate contract?} Its unit is a review occurrence,
with fixed inputs, rules, and acceptance conditions. DGF-Bench measures this question in a
synthetic environment. \emph{Second, can that execution remain dependable across a project?}
Its unit is the complete route, including the evidence and conditions passed between gates.
Route success and repeated trajectories address parts of this question. \emph{Third, does
deployment reduce the human work required to deliver the governed output?} Its unit is an
operating population over a stated period. Answering it requires a labor account, including
work added outside the nominal review team. This third quantity is not present in an API trace.

The units cannot be exchanged without assumptions. Five thousand scored gates are not five
thousand independent companies. A successful run on a generated dossier is not a measured
number of hours saved. A refusal may be a successful review even though the underlying project
does not proceed. Defining these distinctions makes a positive replacement claim more precise:
software can execute an accepted task while humans remain responsible for setting its scope,
maintaining its infrastructure, or resolving a subset of exceptions.

\begin{table}[!htb]
\centering\small
\caption{Claim structure and evidence required. A result at one level does not supply the
missing observation at another.}
\label{tab:claim-structure}
\begin{tabularx}{\linewidth}{@{}P{2.8cm}LL@{}}
\toprule
Claim & Observable test & Evidence in this paper\\
\midrule
Contract execution & Decision, findings, actions, evidence, and mandate meet fixed criteria & Original runs and deterministic control\\
Workflow reliability & Every required gate succeeds; repeated trajectories remain successful & Complete-route scores and 135 additional runs\\
Net task substitution & Accepted output requires fewer total human hours & Accounting conditions and illustrative calculations\\
Ecosystem workforce reduction & Comparable governed output and quality with a lower complete FTE account & Dated hypothesis and a measurement protocol\\
\bottomrule
\end{tabularx}
\end{table}

The central thesis is consequently about the implementation of a function. Organizations may
still require architecture assurance, a security judgment, or evidence of operational readiness.
That continuing requirement does not establish that every dossier needs the same manual
execution. At the same time, describing a function as a contract does not make its evidence
available or resolve policy disagreement. The proposed engineering work is to turn the subset
with adequate information and delegated authority into executable review services, and then
measure the cost of the boundary that remains.

The paper proceeds from the contract to the labor account, then to its FDE implementation and
experimental evidence. We conclude with a deployment test for the workforce hypothesis.
This manuscript synthesizes the longer \emph{The Last Human Gate} and replaces the previous
empirical companion in the repository. The versions share their experimental record;
the expanded exposition reports no additional model calls or enterprise observations.

%% file: figures/dgf_routes.tex
\begin{figure}[!htb]
\centering
\begin{tikzpicture}[
  x=2.5cm,y=1cm,
  gate/.style={draw=black!55,rounded corners=2pt,minimum height=0.64cm,
    text width=2.08cm,align=center,inner sep=2pt,font=\small},
  consolidate/.style={gate,fill=black!8,draw=black!75},
  flow/.style={-{Stealth[length=1.5mm]},draw=black!65,semithick},
  route/.style={anchor=west,font=\small}
]
\node[route] at (-0.46,0.62) {\textbf{BUY} \quad Purchase a solution};
\node[gate,fill=blue!5] (b1) at (0,0) {Procurement};
\node[gate,fill=blue!5] (b2) at (1,0) {Legal};
\node[gate,fill=blue!5] (b3) at (2,0) {Compliance};
\node[gate,fill=blue!5] (b4) at (3,0) {Security};
\node[gate,fill=blue!5] (b5) at (4,0) {IT};
\node[consolidate] (b6) at (5,0) {General};
\foreach \a/\b in {b1/b2,b2/b3,b3/b4,b4/b5,b5/b6}
  \draw[flow] (\a) -- (\b);

\node[route] at (-0.46,-0.73) {\textbf{INTEGRATE} \quad Connect existing systems};
\node[gate,fill=teal!6] (i1) at (0,-1.35) {IT};
\node[gate,fill=teal!6] (i2) at (1,-1.35) {Architecture};
\node[gate,fill=teal!6] (i3) at (2,-1.35) {Security};
\node[gate,fill=teal!6] (i4) at (3,-1.35) {Legal};
\node[gate,fill=teal!6] (i5) at (4,-1.35) {Compliance};
\node[consolidate] (i6) at (5,-1.35) {General};
\foreach \a/\b in {i1/i2,i2/i3,i3/i4,i4/i5,i5/i6}
  \draw[flow] (\a) -- (\b);

\node[route] at (-0.46,-2.08) {\textbf{BUILD} \quad Develop an application};
\node[gate,fill=orange!7] (d1) at (0,-2.7) {IT};
\node[gate,fill=orange!7] (d2) at (1.25,-2.7) {Architecture};
\node[gate,fill=orange!7] (d3) at (2.5,-2.7) {Security};
\node[gate,fill=orange!7] (d4) at (3.75,-2.7) {Tech\\Readiness};
\node[consolidate] (d5) at (5,-2.7) {General};
\foreach \a/\b in {d1/d2,d2/d3,d3/d4,d4/d5}
  \draw[flow] (\a) -- (\b);

\node[font=\small,align=center] at (2.5,-3.48)
  {\textbf{At each gate:} inspect the dossier; record findings, evidence, actions, and a decision.};
\end{tikzpicture}
\caption{Three DGF workflows evaluated in DGF-Bench. A project follows the route matching its
need. Arrows show review order and information handoffs; General consolidates preceding reviews.
The benchmark runs all scheduled gates, including after a refusal. These are example governance
routes, not automatic approvals or a universal enterprise standard.}
\label{fig:dgf-routes}
\end{figure}
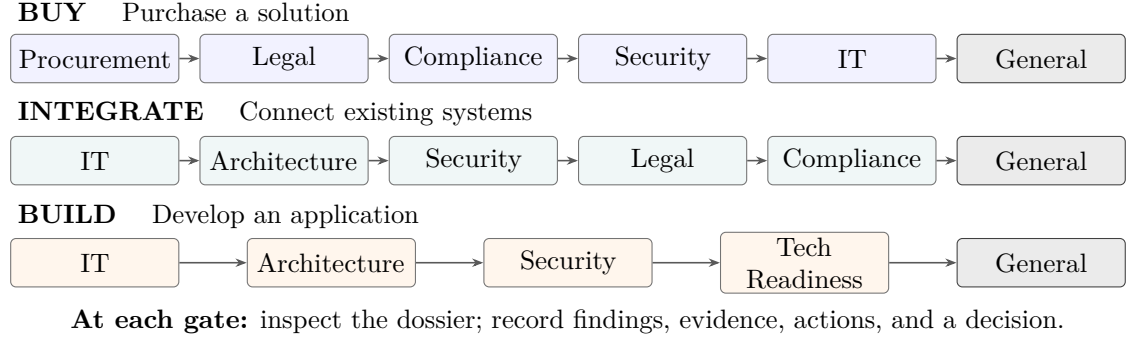

%% file: sections/01_related_work.tex
\section{Related work and the evaluation boundary}
\label{sec:related}

\subsection{From task automation to the organization of work}

Task-based accounts separate the displacement of existing work from changes in demand and
the creation of new tasks \citep{acemoglu2019}. That distinction matters for governance:
an automated review can remove execution hours while the organization expands its project
portfolio or assigns the released capacity elsewhere. We study the labor required for a
comparable governed workload, before interpreting actual employment. This fixes an accounting
question rather than assuming that every efficiency gain becomes a redundancy.

Automation can also leave difficult monitoring and intervention work behind
\citep{bainbridge1983}. A routine-case success rate does not say whether the remaining cases
are more expensive than average. The selection identity in Section~\ref{sec:substitution}
makes that concern explicit through the share of baseline labor located in exceptions.
Sequential task assignment and complementary quality requirements motivate treating handoffs
and bottlenecks as part of the process \citep{demirer2026,gans2026}. Our contribution applies
these questions to a defined governance task and reports the quantities needed to move from
an execution claim to a net-work claim. The equations are accounting relationships and
conditional thresholds, not an alternative macroeconomic model.

\subsection{Governance gates and executable process obligations}

Stage-gate organization predates language-model agents \citep{cooper1990}. It supplies a
recognizable decision boundary: a proposal arrives with evidence, reviewers assess it, and
a decision controls the next step. Business-process compliance research studies how obligations
can be represented and monitored during execution \citep{ly2015}. Agentic business-process
management places agents within a broader process-engineering agenda \citep{calvanese2026}.
These traditions provide the process foundations for the task-substitution question studied here.

The question here is what must hold when execution of such a boundary is transferred from
people to software. A review can include both a deterministic policy kernel and interpretation
of evidence whose form is less regular. The kernel may be compiled into ordinary software;
an agent may recover candidate facts, investigate inconsistencies, or assemble a justified
record. The experiment supplies formalized policies and structured facts. Its rules control
identifies the executable decision kernel; the agent runs measure alternative implementations
of the review contract under that same information condition. The contribution of agents to
recovering less structured facts is a separate comparison specified in Section~\ref{sec:implications}.

\subsection{Enterprise-agent benchmarks}

WorkArena evaluates browser-based tasks on enterprise software, using ServiceNow as its
environment \citep{drouin2024}. ITBench concerns operational IT scenarios in site reliability,
compliance and security operations, and financial operations \citep{itbench2025}.
$\tau$-bench evaluates policy-guided interactions among tools, agents, and simulated users,
and explicitly studies consistency across repeated trials \citep{taubench2024}.
These benchmarks establish important neighboring evaluation problems. Table~\ref{tab:related}
locates DGF-Bench by task and acceptance boundary, rather than comparing incompatible scores.

\begin{table}[!htb]
\centering\small
\caption{Neighboring evaluation tasks. The descriptions concern the cited releases; no
cross-benchmark performance ordering is inferred.}
\label{tab:related}
\begin{tabularx}{\linewidth}{@{}P{2.4cm}LL@{}}
\toprule
Benchmark & Central task & Relation to this study\\
\midrule
WorkArena & Execute knowledge-work tasks through enterprise web interfaces & Interface operation can support evidence collection, but is not the gate-review endpoint here\\
ITBench & Address operational IT automation scenarios & Operational execution and remediation differ from reviewing whether a project meets a governance contract\\
$\tau$-bench & Use tools while following domain policies in a user interaction & Motivates repeated reliability and policy compliance; DGF-Bench follows specialist reviews through a dossier route\\
DGF-Bench & Produce a disposition with findings, actions, observed evidence, and authorization & Measures complete review records and whole-route success under explicit supplied policies and facts\\
\bottomrule
\end{tabularx}
\end{table}

The distinction is consequential. A governance agent that asks for a recovery test has
completed one kind of work; an operational agent that repairs the service and demonstrates
recovery has completed another. DGF-Bench scores the former contract and does not take credit
for the latter. Likewise, the current review tools expose structured facts rather than
requiring a browser agent to find every fact in a production interface. These choices define
the measured task; they prevent interpreting a higher percentage as superiority over an
agent evaluated on a broader or less prepared task.

\subsection{Evidence quality and useful agent comparisons}

Work on attribution and citation quality separates a correct answer from a supported one
\citep{gao2023,rashkin2023}. The separation is central to an auditable review: the final
disposition and the record supporting it can fail independently. Our frozen strict endpoint
requires observed sources and conforming excerpts. A separate structural audit tests how
sensitive the conclusions are to field order and numeric representation. It does not convert
lexical or structural matching into an expert judgment about every premise.

Agent evaluation also benefits from simple baselines and explicit resource accounting
\citep{kapoor2024}. We therefore report the deterministic control, failed attempts, inference
costs, and repeated trajectories together with model scores. In this study the strongest
claim is that specified review tasks can execute in software and that their reliability can
be inspected at several levels. Establishing the incremental value of an agent for less
structured evidence requires the matched-information extension in Section~\ref{sec:implications}.
The labor account then supplies a further test: whether the resulting system transfers work
without recreating equivalent hours in verification, exceptions, or maintenance.

%% file: sections/02_contract.tex
\section{A gate as a substitutable contract}
\label{sec:contract}

A DGF is an organization's composition of review checkpoints. It is not a universal sequence:
enterprises can merge, split, or reorder gates while preserving their obligations and information
dependencies. Typical tasks include supplier due diligence, contractual checks, technical
architecture review, security assessment, readiness verification, and final consolidation.
We study three example routes, Buy, Integrate, and Build; the argument applies to the review
contract, not to these particular names.

\input{sections/02_gate_guide}

\subsection{The review contract}

\begin{definition}[Review contract]
A gate contract specifies permitted evidence and tools, a versioned policy, an authorization
mandate, and the conditions for accepting an output. The output contains a disposition,
findings, required actions, supporting evidence, and an authorization record. Its accepted
form and substantive requirements are fixed before evaluating a proposed replacement.
\end{definition}

In DGF-Bench the dispositions are \texttt{GO}, \texttt{GO\_WITH\_RESERVATIONS}, \texttt{REWORK},
\texttt{SUSPENSION}, and \texttt{NO\_GO}. A refusal can be a correct completed review. Conversely,
an approval is unsuccessful if it omits a mandatory defect or exceeds the reviewer's mandate.
Completing the review is distinct from implementing the remediation it requests. The same
boundary must apply to human and software work when estimating substitution.

\subsection{What is held fixed when execution changes}

A replacement is evaluated on a declared population: for example, supplier reviews below a
specified purchasing threshold, or architecture reviews of services using a defined platform.
The population includes difficult cases, missing evidence, and legitimate refusals. Excluding
these after observing failures would change the task being automated. A narrower deployment
can still be useful, but its scope and excluded work must be reported explicitly.

The contract also identifies a review's start and end. Reading a complete dossier and issuing
an opinion is a different work package from obtaining that dossier, negotiating a clause, or
verifying a corrected deployment. A study of the first package can establish execution of that
package. It cannot silently count the other packages as replaced. Conversely, the continued
need for remediation does not mean that issuing the review remains a human task. These
distinctions permit substitution to be measured at useful intermediate boundaries.

Versioning makes the comparison inspectable. A decision record should identify the policy,
source state, and mandate under which it was produced. A later policy change may justify a
different decision on the same project; it does not retroactively make the earlier system
inconsistent. If documents change during a review, the contract must specify whether the
agent freezes an evidence version, refreshes affected checks, or reopens the gate. Such
requirements apply to the whole implementation, including its tools and controller.

\subsection{Three conditions for a replacement}

\textbf{Information.} The permitted observations must contain what the policy needs, or the
contract must allow a request for information or abstention. A model cannot recover a missing
fact merely by generating a plausible answer. Source accessibility matters as much as physical
file existence; the agent may lack permission to read a relevant record.

\textbf{Validity.} The system must produce outputs satisfying the review's quality and service
requirements on its declared operating population. These include evidence support and the
handling of uncertainty, not just a disposition matching a label. A machine-readable schema
can reject malformed records but cannot by itself prove the substantive decision correct.

\textbf{Authority.} A valid mandate must permit the system's action. An agent's proposal and
the mechanism that makes a commitment effective are separate. Deterministic enforcement can
block forbidden commitments even when an agent requests them; a final authorized output does
not establish that the model respected authority unaided. Delegating bounded execution does
not eliminate the organization's accountability or resolve deployment-specific legal requirements.

These conditions describe functional replacement. Net substitution of human work additionally
requires a reduction in all human hours used to deliver that function, as derived next. Removing
a routine reviewer while adding equivalent effort to project teams or support providers moves
the work rather than reducing it.

Quality and coverage must be reported together. An implementation that escalates every case
may avoid unauthorized approvals while performing almost none of the review itself. At the
other extreme, an implementation that approves every case completes many transactions but
does not preserve the governance function. An abstention is valid only when the contract
permits it for the observed information or uncertainty; the resulting human investigation
still belongs in the labor account. Correct refusals, permitted abstentions, unresolved
reviews, and technical failures therefore require distinct outcome labels.

For an engineering comparison, acceptance is assessed on the final record and the path by
which it became effective. If an enforcement tool rejects a forbidden request, the system
has preserved that authority boundary. The rejected request remains evidence about the
model's behavior. Similarly, a corrected evidence excerpt can make the final record usable
while revealing a need for extra verification. Reporting only the final label would conceal
these different mechanisms and their operating costs.

\subsection{An information obstruction}

Let $I(x)$ denote all observations permitted in an evaluation condition for underlying case $x$,
and $V(x)$ the set of outputs accepted by its fixed contract. Policy and mandate are held fixed.

\begin{proposition}[Indistinguishable cases]
If two admissible cases satisfy $I(x_1)=I(x_2)$ but
$V(x_1)\cap V(x_2)=\varnothing$, no system using only $I$ can guarantee a valid output on both.
\end{proposition}
\begin{proof}
Equal observations imply equal output distributions, including for a deterministic system.
Guaranteed validity on both cases would require one distribution to assign probability one
to both disjoint accepted sets, which is impossible.
\end{proof}

This elementary observation is an operational test for an automation boundary, not a new general
impossibility theorem. Our offline counterexample makes it concrete: two variants have identical
extracted text in all 26 Word documents, but changing the underlying supplier due-diligence
status, recorded in CSV evidence, changes the fixed reference decision from GO to REWORK.
A condition restricted to those Word
texts therefore cannot reproduce both reference decisions. Allowing the CSV, or changing the
contract to accept a justified request for missing information, changes the task. For the fixed
reference test, the accepted disposition sets are disjoint. This obstruction applies to a human
reader as well as a model; it identifies a deficient information interface.

\subsection{Why the DGF is a candidate}

Table~\ref{tab:candidate} turns the claim that DGF is suitable for automation into properties
that can be checked in an organization. These are engineering advantages when present, not
guarantees that every enterprise already has clean rules and accessible evidence.

\begin{table}[!htb]
\centering\small
\caption{Properties that make a governance workflow a candidate for task substitution.}
\label{tab:candidate}
\begin{tabularx}{\linewidth}{@{}P{3.1cm}LL@{}}
\toprule
Property & Potential advantage & Condition that can defeat it\\
\midrule
Recurring dossiers and outputs & Reusable extraction and decision interfaces & Material facts remain tacit or inaccessible\\
Explicit policies and mandates & Executable checks and bounded delegation & Conflicting rules or discretionary commitments\\
Recorded gate handoffs & Shared evidence and traceable dependencies & Work is repeated or lost between teams\\
Repeated review populations & Integration cost spread across cases & Low volume or frequent policy change\\
Observable acceptance criteria & Regression tests and monitored service & Evaluation rewards format instead of valid work\\
\bottomrule
\end{tabularx}
\end{table}

These properties make the workflow an engineering candidate because its obligations can be
named and inspected. They do not require a single agent to perform every operation. A rule
engine can implement an explicit threshold; an agent can investigate a narrative discrepancy;
an authorization service can enforce the resulting scope. The relevant comparison concerns
whether the combined implementation executes the accepted work at the required quality and
resource cost. The benchmark's deterministic control is consequently informative about the
formalized decision kernel, even though it cannot establish the benefit of language models
for extracting facts or resolving ambiguous policies.

At route level, replacement also requires compatible handoffs. An upstream decision may be
valid only subject to a condition that a later gate must verify. Dropping that condition changes
the governed output. Merging two software services is acceptable if both obligations remain
enforced; skipping a required check is not. Because failures and missing evidence can propagate,
whole-route evaluation is necessary alongside individual gate scores. Multiplying average
gate success rates would require assumptions about dependence that this benchmark does not
establish.

%% file: sections/02_gate_guide.tex
\subsection{The work behind the checkpoints}

A project owner assembles a dossier describing the proposed change: its need, users, budget,
technical design, supplier commitments, and readiness evidence. A specialist gate examines the
parts relevant to its policy and returns a recorded opinion. The same dossier can therefore
receive different findings from Architecture, Security, and Legal without those reviews being
duplicates. General combines their effective opinions with its own governance checks.
Table~\ref{tab:gate-guide} explains the eight families represented in the experiment.

\begin{table}[!htb]
\centering\small
\caption{What each benchmark gate reviews. Questions summarize the supplied policies; listed
records and checks are examples, not exhaustive definitions of the professions.}
\label{tab:gate-guide}
\begin{tabularx}{\linewidth}{@{}P{2.45cm}LL@{}}
\toprule
Gate & Review question & Examples of material examined\\
\midrule
Procurement & Does the selected supplier satisfy purchasing requirements?
& Offers, mandatory criteria, due diligence, sanctions, and three-year cost\\
Legal & Are the contractual prerequisites and signing authority present?
& Data-processing agreement, liability, exit assistance, and signing authority\\
Compliance & Does the proposal satisfy the applicable compliance checks?
& Required impact assessment, residency, regulatory mapping, and audit trail\\
Security & Does the design provide the protections required for its exposure and data?
& Identity controls, private endpoints, monitoring, and vulnerability findings\\
IT & Does the proposal fit the technology and operating environment?
& Catalog status, existing capability, capacity, support ownership, and change records\\
Architecture & Is the proposed system consistent with the declared design constraints?
& Network overlap, interfaces, data ownership, latency, and reversibility\\
Tech Readiness & Is the evidence of operational readiness sufficient?
& Load, restore and recovery tests, runbook, rollback, and handover\\
General & What governance disposition follows from the project and specialist reviews?
& Effective upstream opinions, budget, strategic alignment, and change plan\\
\bottomrule
\end{tabularx}
\end{table}

The five dispositions express different completed review outcomes. \texttt{GO} permits
proceeding within the reviewed scope. \texttt{GO\_WITH\_RESERVATIONS} retains conditions
alongside that permission. \texttt{REWORK} requests changes to the submitted proposal;
\texttt{SUSPENSION} waits for a prerequisite; \texttt{NO\_GO} rejects proceeding as submitted.
The policy and mandate determine which outcome is admissible. The experimental schedule
collects every review in the route, including after a blocking opinion, to evaluate the
remaining specialist work and final consolidation.

Four responsibilities clarify the unit of substitution: preparing the dossier, examining
it, making an authorized decision effective, and carrying out the resulting actions.
An organization may allocate several of these responsibilities to one person, or distribute
them across project teams, specialists, and decision owners. In the experiment, the agent
performs the specified review and can use authorized simulated actions. The generator supplies
the dossier; the trace records requested remediation and commitments. A deployment's labor
account follows each responsibility wherever it is performed.

For example, a readiness review may identify a draft runbook, require its completion, and
record an authorized conditional decision. The review has then produced its required output;
the runbook still has an owner and an open completion condition. The Falcon trace in
Section~\ref{sec:benchmark} records precisely this separation. It explains how an agent can
replace a review task while another person or system continues the operational work triggered
by that decision. Extending substitution to that operational work requires its own accepted
output and authority.

%% file: sections/05_labor.tex
\section{When agents replace human gate work}
\label{sec:substitution}

\paragraph{A complete account.}
At fixed case volume $\lambda>0$, let $h_0>0$ be mean baseline handling hours. A frozen routing
rule sends a fraction $q$ of cases to substantive human exception handling. Let $h_E$ be
complete human hours on those cases, $h_R$ ordinary-path human review, and $h_W$ additional
rework per case. Recurring human upkeep is $B_H$ in the baseline and $B_A$ after deployment.
Then
\begin{equation}
 L_H=\lambda h_0+B_H,\qquad
 L_A=\lambda\{qh_E+(1-q)h_R+h_W\}+B_A.
 \label{eq:short-hours}
\end{equation}
Exceptions include their preliminary review; ordinary review includes mandatory signatures
and expected sampled audits. Rework, downstream repair, supplier effort, evaluation, and
policy maintenance are allocated once across the common process boundary. Changing a
department or job title does not remove the hours. Initial implementation is reported
separately for a transition-cost calculation; recurring adaptation belongs in $B_A$.

The fraction of cases retained by humans is generally different from the fraction of work
retained. Let $T$ be baseline handling time, $J$ indicate exception routing, and
$h_0=\mathbb E[T]$. Define the residual's baseline labor share
$\phi=\mathbb E[TJ]/h_0$, its effort multiplier
$\eta=h_E/\mathbb E[T\mid J=1]$, and the ordinary group's retained review fraction
$\mu=h_R/\mathbb E[T\mid J=0]$. Assume both conditional baseline means are positive and
$0<q<1$; Equation~\eqref{eq:short-hours} covers the other cases. With
$r=h_W/h_0$ and $b_s=B_s/(\lambda h_0)$, the workload ratio is
\begin{equation}
 \rho:=\frac{L_A}{L_H}
 =\frac{\phi\eta+(1-\phi)\mu+r+b_A}{1+b_H}.
 \label{eq:short-rho}
\end{equation}
This is an accounting identity, not an estimated effect of agents.

Table~\ref{tab:labor-measurement} identifies what must be recorded to use the identity in an
organization. A common case identifier links human activities before and after routing. The
comparison follows the same work across teams and suppliers, rather than assuming that an
hour removed from one department has disappeared from the process.

\begin{table}[!htb]
\centering\small
\caption{Operational measurements behind the labor account. Conditional means use the same
frozen routing partition; time categories must not overlap.}
\label{tab:labor-measurement}
\begin{tabularx}{\linewidth}{@{}P{2.5cm}LL@{}}
\toprule
Quantity & Required observation & Interpretation\\
\midrule
$\lambda,h_0$ & Review visits and baseline hours & Scale and baseline intensity\\
$q,\phi$ & Routed cases and their baseline hours & Case coverage versus retained labor\\
$h_E,\eta$ & Complete exception effort & Assistance or inflation on the residual\\
$h_R,\mu$ & Ordinary review and audit effort & Human work remaining on the standard path\\
$h_W$ & Additional correction and rework & Work created beyond the recorded paths\\
$B_H,B_A$ & Allocated recurring support & Human upkeep in both regimes\\
\bottomrule
\end{tabularx}
\end{table}

\begin{proposition}[Selection and the substitution frontier]
\label{prop:short-frontier}
Let $T\ge0$ be integrable with positive mean, and let
$e(t)=\Pr(J=1\mid T=t)$. Under the definitions above,
\begin{equation}
 \phi=q+\frac{\operatorname{Cov}(T,e(T))}{h_0}.
 \label{eq:short-selection}
\end{equation}
Thus routing that is nondecreasing in baseline effort retains at least as much baseline
labor as its share of cases. For any target fractional workload reduction $d\in[0,1]$,
achieving $L_A\le(1-d)L_H$ is equivalent to
\begin{equation}
 \phi\eta+(1-\phi)\mu+r+b_A\le(1-d)(1+b_H).
 \label{eq:short-target}
\end{equation}
\end{proposition}
\begin{proof}
Conditional expectation gives $\mathbb E[TJ]=\mathbb E[Te(T)]$. Expanding the covariance
and using $\mathbb E[e(T)]=q$ proves Equation~\eqref{eq:short-selection}. For independent
copies $T,T'$, twice this covariance is
$\mathbb E[(T-T')(e(T)-e(T'))]$, which is nonnegative when $e$ is nondecreasing.
Splitting $h_0$ between the two routing groups yields
$qh_E/h_0=\phi\eta$ and $(1-q)h_R/h_0=(1-\phi)\mu$.
Substitution into Equation~\eqref{eq:short-hours} proves
Equation~\eqref{eq:short-rho}; multiplication by its positive denominator gives the target
condition.
\end{proof}

The practical consequence is that an agent success rate cannot be read as the percentage
of human labor removed. For example, 20\% of cases retain 40\% of baseline labor when their
baseline handling time is twice the overall mean. If their post-deployment effort rises by
50\%, those exceptions alone consume 60\% of baseline handling hours. This is the
mechanism by which a small exception queue can prevent a large workforce reduction.

The selection distinction is measurable before interpreting any automation score. Apply a
candidate routing rule to cases for which baseline effort is recorded, then compare the
hours associated with its residual against total hours. This estimates $\phi$ for that
partition; deployment is still needed to estimate $\eta$ and $\mu$. If the routing rule or
case mix changes, the partition must be remeasured or reweighted. Holding $q$ constant is
insufficient: two queues containing the same number of cases can retain different shares
of baseline labor. Random auditing can add another reason for human involvement without
making those selected cases inherently difficult.

\paragraph{Reliability consumes the same budget.}
If erroneous accepted outputs occur on a fraction $p$ of all cases and cause additional human
repair averaging $h_C$ hours conditional on such an error, their contribution to $r$ is
$p h_C/h_0$, counted once.
This links execution reliability to the labor target, but does not make low repair cost a
quality guarantee. An undetected harmful error is unacceptable even if it consumes no repair
hours. Strict benchmark failures, human escalations, and erroneous accepted outputs are
different events; the experiments do not measure deployment values of $q$ or $h_C$.

\paragraph{Coverage does not determine required labor.}
The extended manuscript's reproducible synthetic example allocates 140 baseline FTE
across nine governance populations, at 120 useful hours per FTE per month. Each population
has 100 pooled review visits, exception fractions from 0.16 to 0.33, and exception baseline
means twice its overall mean. The workload-weighted exception fraction is 23.43\%, while
the groups retain 46.86\% of baseline labor. With these partitions unchanged, the
calculated requirements are shown in Figure~\ref{fig:labor}.
The scenario parameters $(\eta,m,h_W,B_A)$ are respectively
$(0.5,0.10,0,0)$, $(1,0.10,0,120)$, $(1,0.10,0.5,120)$, and
$(1.5,0.35,2,240)$. Here $h_R=m h_0$, so $m$ differs from the group-relative $\mu$;
$h_W$ is hours per case and $B_A$ hours per population per month. Baseline upkeep is zero.
The paired population vectors are baseline FTE $(10,10,30,20,10,25,15,10,10)$ and
$q=(0.22,0.33,0.18,0.16,0.22,0.24,0.26,0.33,0.33)$.
\begin{figure}[!htb]
\centering
\includegraphics[width=0.88\linewidth]{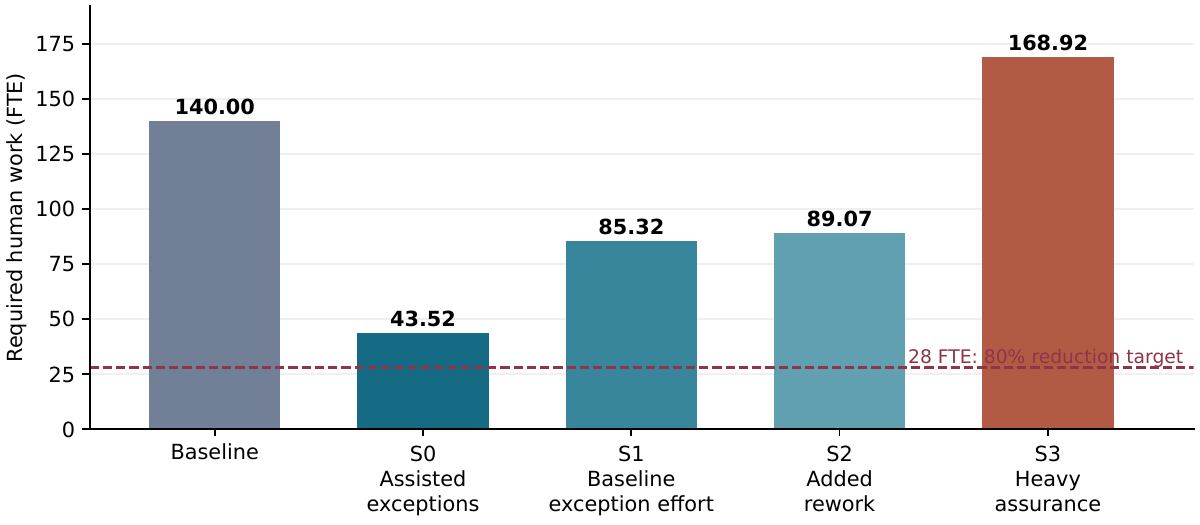}
\caption{Illustrative human-work requirements with the same case partition: baseline 140 FTE;
four operating assumptions produce 43.52, 85.32, 89.07, and 168.92 FTE. The dashed line is the
author's 80\% reduction target, not a fitted forecast. Parameters and calculation are released.}
\label{fig:labor}
\end{figure}
These are sensitivity calculations using declared assumptions, not measurements or
forecasts. The comparison establishes a useful non-identification result: even fixed
coverage and a fixed baseline case partition do not determine the sign of the labor
effect. Their associated quality is not estimated. A deployment must measure both the
contract outcomes and the complete account to identify whether execution has actually
been substituted at acceptable quality.
None of these four scenarios reaches an 80\% reduction. With $\phi\simeq0.4686$ and no other
residual burden, that target already requires $\eta\le0.20/\phi\simeq0.427$.
Thus the forecast requires further reduction in exception effort or in its retained labor
share, even before allowing for ongoing review and support.

\paragraph{A worked calculation.}
The Architecture population supplies a self-contained example of the normalization. Its
30 baseline FTE give $h_0=30\times120/100=36$ hours per visit. With $q=0.18$ and residual
baseline effort twice the overall mean, an exception originally requires 72 hours. The
ordinary group's baseline mean must therefore be
$(36-0.18\times72)/0.82\simeq28.0976$ hours. The two conditional means reproduce the same
36-hour baseline; assigning 36 hours to both groups would erase the selection effect.

Under the first scenario, $\eta=0.5$ reduces exception effort to 36 hours, while
$m=0.10$ gives ordinary review of 3.6 hours. Mean human effort is
$0.18\times36+0.82\times3.6=9.432$ hours, hence
$100\times9.432/120=7.86$ FTE. Under the second scenario, exceptions retain 72 hours:
case work requires 13.26 FTE, and 120 monthly support hours add one FTE, producing 14.26.
The third scenario adds 0.5 hours on each of 100 visits, or $50/120$ FTE, producing 14.68
after rounding. Coverage is unchanged throughout. These calculations use the declared
synthetic parameters and illustrate how effort and upkeep, rather than the case fraction
alone, determine the result.

The first scenario also illustrates why review ratios need their denominator stated:
$m=3.6/36=0.10$, whereas $\mu=3.6/28.0976\simeq0.1281$. Substituting $m$ for $\mu$ in
Equation~\eqref{eq:short-rho} would undercount review. The group-relative form makes
selection explicit; the dimensional equation is usually the simplest way to calculate
an observed operating total.

\paragraph{Volume and the support boundary.}
Recurring support is not automatically proportional to visits. With otherwise fixed
conditions, more visits spread the same $B_A$ over more cases and lower its normalized
burden $b_A$. Additional policy variants, incidents, or service requirements can instead
increase support. The model therefore retains the recorded hours; it does not assume that
scale makes them vanish. Shared services must be allocated once across the evaluated
workflow, including supplier work needed to produce and maintain usable evidence.

Repeated reviews also require a consistent convention. A failed submission may generate a
new visit or additional rework within its original visit, but the same activity cannot be
counted both ways. If automation changes visit frequency, equal visits no longer mean equal
governed output. The deployment comparison should then follow complete cases to the same
endpoint and account for each regime's visits. The fixed-volume frontier isolates the
effort mechanism; it does not remove the need to measure changed demand or repeated work.

\input{sections/05_route_scale}

\paragraph{From required work to staffing.}
At fixed useful capacity $K>0$, required workload-equivalent FTE equal $L/K$. Under the
additional assumptions of interchangeable workers, target utilization $0<u\le1$, no separate
coverage constraint, and staffing chosen to minimize labor cost, required headcount is
$\lceil L/(uK)\rceil$. Reducing required hours then removes posts when this integer crosses
a staffing threshold. Redeployment or increased demand can absorb the released capacity;
they do not make execution of the old tasks indispensable again. For a finite collection
of gates with bounded visit counts, vanishing human case effort leaves a support floor
$B_A/K$; zero total human labor further requires that this support work be eliminated or
automated. These are explicit conditions for displacement, not a claim that current
benchmark performance has already satisfied them.

%% file: sections/05_route_scale.tex
\subsection{Route volume and the allocation of recurring support}
\label{sec:route-scale}

The same platform can have different effects on the human work required by different routes.
The extended manuscript's second illustration makes this visible using one specialist team
with 100 monthly review visits: five Integrate and 95 Build. Each visit has a baseline mean
of 24 hours, so the combined 2,400 hours correspond to 20 FTE at the same 120-hour convention.
These are selected operating assumptions, with review visits as the unit; the route names
do not imply that their observed enterprise frequencies have been estimated.

Both routes retain $h_R=2.4$ ordinary review hours and $h_W=0.5$ additional rework hours per
visit. Each is allocated $B_A=66$ recurring support hours per month and $B_H=0$.
Define $\kappa=\mathbb E[T\mid J=1]/h_0$, so $\phi=q\kappa$ and
$h_E=\eta\kappa h_0$. Table~\ref{tab:route-scale} supplies the remaining inputs and results.
The conditional ordinary baseline means are fixed by the common 24-hour mean; for Integrate,
this gives $(24-0.60\times1.4\times24)/0.40=9.6$ hours. Thus the exception partition and
baseline remain coherent.

\begin{table}[!htb]
\centering\small
\caption{Synthetic route-specific labor requirements on the same platform. The final row
changes only Integrate's visit volume; it is an alternative scale calculation, not a third
route added to the 20-FTE population.}
\label{tab:route-scale}
\begin{tabular}{@{}lrrrrrr@{}}
\toprule
Scenario & Visits & $q$ & $\kappa$ & $\eta$ & Support h/visit & $\rho$\\
\midrule
Integrate & 5 & 0.60 & 1.4 & 1.0 & 13.20 & 1.451\\
Build & 95 & 0.14 & 2.0 & 0.8 & 0.69 & 0.360\\
Integrate at higher volume & 95 & 0.60 & 1.4 & 1.0 & 0.69 & 0.930\\
\bottomrule
\end{tabular}
\end{table}

For Integrate at five visits, Equation~\eqref{eq:short-hours} gives
$L_A=5\{0.60\times33.6+0.40\times2.4+0.5\}+66=174.10$ hours against 120 baseline
hours. Build requires $L_A=95\{0.14\times38.4+0.86\times2.4+0.5\}+66=820.30$ hours
against 2,280. Their combined 994.40 hours equal 8.29 FTE and a ratio of 0.414. The combined
reduction therefore coexists with an increase on the low-volume route. At 95 Integrate visits,
the same case effort and 66 support hours give a ratio of 0.930: the larger denominator spreads
the fixed support term over more accepted review work.

This calculation identifies a concrete FDE implementation choice. Reusing maintained
connectors, policy tests, and evidence interfaces can distribute recurring work across a
larger review population, while specialized exceptions still require their own effort.
The operating ledger must record that allocation once and retain route-level results alongside
the total. This lets an organization identify where task substitution already reduces required
work and where a different operating design is needed. Quality and service remain separate
acceptance conditions; the table evaluates the labor consequences of the stated assumptions.

%% file: sections/04_deployment.tex
\section{The FDE implementation: engineering the whole workflow}
\label{sec:fde}

Forward deployed engineers (FDEs) work alongside organizations to turn software capabilities
into deployed systems; embedded implementation is explicit in the role described by
\citet{palantirfde2026}. Here the role has a concrete purpose: construct the information,
policy, evidence, and authorization interfaces that let an accepted review be completed without
requiring a person to reconstruct it. The job title itself does not establish an advantage.
The engineering contribution is an operating arrangement: which tasks transfer to software,
which commitments that software may make, and who handles the work it cannot complete.

Figure~\ref{fig:implementation} separates responsibilities. Agents can investigate a dossier,
extract candidate facts, interpret narrative material, and propose findings. Where policies
are executable, a deterministic service evaluates them. Evidence services attach copied source
values and locations rather than asking a model to reproduce quotations. Mandate checks govern
commitments. The system escalates cases for which its evidence or delegated scope is inadequate.
These are proposed deployment components, not a claim that this complete architecture was
evaluated by the current benchmark.

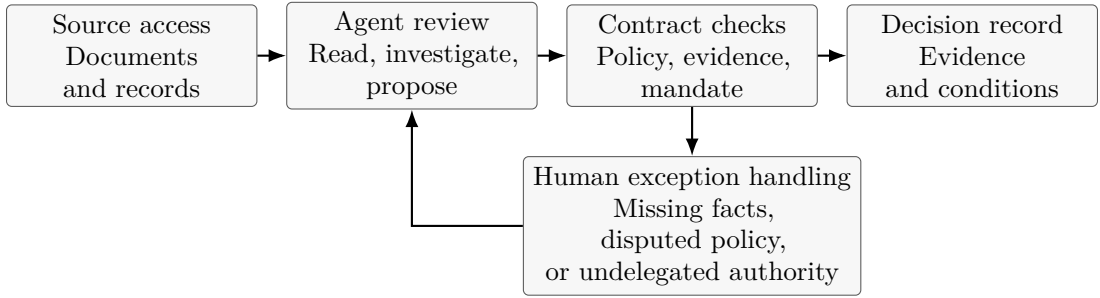
\begin{figure}[!htb]
\centering
\begin{tikzpicture}[>=Latex,box/.style={draw=black!65,rounded corners=2pt,fill=black!3,
  text width=3.05cm,minimum height=1.25cm,align=center,font=\small},node distance=0.38cm]
\node[box] (sources) {Source access\\Documents\\and records};
\node[box,right=of sources] (agent) {Agent review\\Read, investigate,\\propose};
\node[box,right=of agent] (check) {Contract checks\\Policy, evidence,\\mandate};
\node[box,right=of check] (output) {Decision record\\Evidence\\and conditions};
\draw[->,thick] (sources)--(agent);
\draw[->,thick] (agent)--(check);
\draw[->,thick] (check)--(output);
\node[box,below=0.65cm of check,text width=4.2cm] (human) {Human exception handling\\Missing facts, disputed policy,\\or undelegated authority};
\draw[->,thick] (check)--(human);
\draw[->,thick] (human.west)-|(agent.south);
\end{tikzpicture}
\caption{Proposed gate implementation. Every human contribution, including support of these
services, belongs in the labor account. Passing a technical check alone is not a certificate
of complete substantive correctness.}
\label{fig:implementation}
\end{figure}

\subsection{Specify the transfer before building the agent}

The first deliverable is a named task boundary agreed with the process owner. It identifies
the input population, accessible systems, output record, service requirements, and decisions
that remain outside the mandate. ``Review an architecture'' is too broad to assess a transfer.
``Evaluate this submitted design against these network and recovery requirements, preserve
unresolved conditions, and return a supported disposition'' identifies observable work.
Implementing the requested design changes is a separate task unless the contract includes it.

An acceptance set should cover approved, deficient, incomplete, conflicting, and out-of-scope
inputs. Changing a material fact should change the required finding or disposition; changing
only field order should not alter its meaning. Missing information must trigger the declared
request or abstention behavior. These proposed tests expose contract defects before a model
score is interpreted as substitution. The policy owner resolves disputed obligations; the
engineer records that resolution rather than allowing a model to silently invent policy.

\subsection{Make evidence usable without reconstructing the dossier}

Source access requires more than a directory of documents. For every material premise, the
implementation identifies its authoritative source, accessible representation, applicable
version, and treatment of conflicting observations. The source counterexample in
Section~\ref{sec:contract} gives a direct acceptance test: the permitted interface must
distinguish cases whose required outputs differ, or permit an explicit request for the
missing information. Preparing a snapshot manually for every project belongs in the human
work account; it cannot be treated as free model input in a deployment comparison.

The evidence service should preserve the observation used by the reviewer and attach its
source location to the finding. A returned record can contain the case identifier, source
version, field path, copied value, and observation time. This addresses the observed
serialization failures without requiring a model to reproduce source text exactly. However,
a copied value can still be irrelevant or incomplete. Checking that an excerpt was observed
and checking that it supports every required premise remain distinct acceptance criteria.

There is a tradeoff between a uniform evidence interface and the cost of maintaining it.
Normalization may reduce repeated interpretation downstream, while connectors, changing
schemas, and access controls add upkeep. Both effects need measurement. The interface should
report inaccessible or stale sources explicitly; substituting a plausible value would hide
the information failure from the next gate.

\subsection{Separate policy execution, commitments, and failure handling}

When policy is executable, the deterministic control supplies a reason to implement that
kernel directly. Agents can propose findings and gather inputs around it. This division
does not transfer responsibility for choosing the policy: responsible owners must approve
its version and exceptions. Acceptance tests should include threshold cases, conflicts
between rules, and cases requiring a refusal. The released benchmark establishes performance
under its supplied policies, not the correctness of an organization's policy conversion.

Authority enforcement needs a separate check of who may commit, what scope is covered, and
which conditions remain open. A model asking for approval is not evidence that approval is
permitted. The rejected conditional-approval requests in the experiment show the practical
importance of that distinction. A denied request must remain denied in the final record;
an executed conditional approval must retain the underlying finding until separate evidence
establishes closure. Repeated calls must not create duplicate commitments.

Failure behavior is part of the same contract. A timeout, exhausted tool budget, unavailable
source, or malformed final record must produce the specified failure state. In particular,
finalization must not silently supply a default approval, as the observed placeholder error
illustrates. More retries may improve completion but consume time and resources; retries
cannot resolve a missing mandate or make an unsupported premise true.

\subsection{Preserve obligations across gate interfaces}

Common interfaces matter even when rollout is incremental. Consider an Architecture review
accepting a design only if a private endpoint is added. Security needs the condition, its
owner, and its fulfillment status; an unqualified approval would change the meaning of the
upstream review. A handoff should therefore preserve the disposition, unresolved findings,
actions, evidence references, policy version, and effective authorization. General must
consolidate that state rather than infer closure from the presence of an approval label.

A useful acceptance test follows one unresolved condition through the entire route. Its
status must change only when the declared closure evidence is supplied. Another deliberately
replays an earlier record after its source or mandate changes: the system should follow the
predeclared revalidation rule rather than silently treat an old decision as current. These
are proposed integration tests, not additional experiments reported here. Their purpose is
to prevent local convenience from creating verification work for downstream teams.

\begin{table}[!htb]
\centering\small
\caption{Proposed engineering deliverables and acceptance checks. Human hours are allocated
once to the account in Section~\ref{sec:substitution}; these are deployment requirements,
not measured savings.}
\label{tab:fde-deliverables}
\begin{tabularx}{\linewidth}{@{}P{2.65cm}LP{3.4cm}@{}}
\toprule
Deliverable & Acceptance criterion & Human-work location\\
\midrule
Source interface & Recover required premises or report their absence; preserve source identity
& Preparation and verification; recurring connector upkeep in $B_A$\\
Policy and mandate service & Apply the specified version; reject unauthorized commitments;
retain unresolved conditions & Ordinary checks in $h_R$; exceptional decisions in $h_E$\\
Evidence and handoff record & Findings retain observed support and their open obligations
through downstream reviews & Repeated review in $h_R$ or $h_E$; additional repair in $h_W$\\
Operating ledger & Record case work and shared support across teams and suppliers
& Case work in $h_E,h_R,h_W$; shared recurring effort in $B_A$\\
\bottomrule
\end{tabularx}
\end{table}

\subsection{Roll out against quality and labor criteria}

A proposed rollout starts with a frozen contract and representative operating cases, then
tests the interfaces without allowing unapproved commitments. Shadow operation can reveal
incorrect decisions and missing evidence, but it adds an agent beside existing work; it
does not by itself demonstrate substitution. Transfer of a bounded task should begin only
under a declared mandate, with the required quality and service criteria fixed beforehand.
The retained human path needs the information necessary to resolve the exception, not merely
an error message that forces the specialist to start again.

Stopping criteria belong in the operating agreement. Unauthorized commitments, unsupported
material findings, loss of open conditions, or failure of a required source trigger the
specified containment and escalation response. Persistent deterioration in quality or
service stops expansion even if more cases appear automated. A change in policy, source
interface, or model configuration requires the relevant acceptance checks before expanding
its use. A validator that escalates everything can satisfy a conservative commitment rule
while transferring no execution; both acceptance and escalation volumes must be reported.

The economic check uses the same population and boundary as the baseline. Record ordinary
review, exception handling, additional correction, and recurring engineering separately,
including project-team and supplier effort. Initial integration is a transition investment;
continued adaptation belongs in $B_A$. A reduction in inference cost does not offset an
unrecorded increase in specialist hours when the claim concerns human-work substitution.
If a task no longer needs case-by-case human execution, its released hours count even when
the organization redeploys the affected staff. If every output needs a lengthy signature
or reconstruction, that work remains in $h_R$ or $h_E$.

These deliverables explain the proposed DGF-first strategy. An FDE implementing an enterprise
agent encounters governance during deployment; the interfaces built to navigate its gates
can also support execution of their reviews. Recurring evidence, policies, and decision
formats create opportunities for reuse. High-volume business operations may nevertheless
justify earlier automation. The claim is a plausible engineering sequence, not an observed
adoption ranking. What ultimately distinguishes substitution is an accepted task delivered
with less total human work, rather than the number of reports generated or the engineer's
job title.

%% file: sections/02_experiment.tex
\section{Evidence: executing governance contracts}
\label{sec:benchmark}

\subsection{Design, dossiers, and information conditions}
DGF-Bench tests the execution of a specified review task: inspect permitted evidence, choose a
disposition, identify findings and required actions, cite the supporting observations, and respect
the authorization contract. A seeded generator constructs 300 fictional projects, 100 each on
Buy, Integrate, and Build routes (Figure~\ref{fig:dgf-routes}). Buy and Integrate each contain six
gates; Build contains five, giving 1,700 gate occurrences per model.
The dossiers include project charters and review requests, architecture diagrams,
technical designs, supplier and legal records, and operational test evidence. Decision-coverage
sampling balances applicable outcomes. Difficulty setting 4 controls generator parameters;
it is not a validated scale of reasoning difficulty. The 300 projects have distinct architecture
and complete decision-fact signatures, although some gate-level patterns recur.

Each project begins with canonical facts: an owner, a business context, requested budget, users,
data classification, architecture, and operational or contractual constraints. Templates derive
the dossier from those facts. A Buy case may expose supplier offers, due-diligence records, and
contract status; a Build case may expose readiness tests, network arrangements, and a runbook.
The environment can mark evidence missing, stale, or inconsistent. The same program defines
the public policy and evaluator reference, so successful execution measures conformity to that
declared policy rather than independently validated professional judgment.

Decision balancing broadens coverage of approval, refusal, and intermediate outcomes. It does
not make projects independent samples of enterprise demand. The released diversity audit finds
267 distinct whole-case finding patterns among the 300 projects; gate-specific relevant-fact
signatures include 142 of 200 Compliance cases and 241 of 300 Security cases. Phase-adjusted
normalized decision entropy ranges from 0.9971 to 1.0000. These statistics describe diversity
under the generator's representation; they do not certify equivalent diversity of organizational
contexts or difficulty. Difficulty 4, for example, sets the pre-balancing architecture IP-overlap
probability to 0.20 and shared-service-principal probability to 0.22. The accepted distribution
also depends on decision-balancing rejection sampling.

The tested information condition supplies both executable policies and an authoritative
structured \texttt{REVIEW\_FACTS} snapshot. Models can read allowed evidence and use simulated
enterprise tools; evaluator-only reference files are inaccessible. The model is the reviewer,
not the generator of the project. Every gate returns one of \texttt{GO},
\texttt{GO\_WITH\_RESERVATIONS}, \texttt{REWORK}, \texttt{SUSPENSION}, or \texttt{NO\_GO}.
Correctly refusing a deficient proposal counts as success. Later gates receive the model's actual
upstream reviews, and General consolidates the route.

The policy supplies objectives, permitted dispositions, finding identifiers, required actions,
and executable rule definitions. A gate prompt gives the relevant project context and output
schema. The agent chooses which allowed sources to inspect, can query simulated system records
or request evidence, and can inspect an authorization mandate. A final structured submission
records its decision, findings, actions, evidence references and excerpts, authorization flag,
and rationale. Simulated actions change only the test environment. A schema-valid submission
can still be substantively wrong, and requesting an action does not establish its completion.

\begin{figure}[!htb]
\centering
\includegraphics[width=\linewidth]{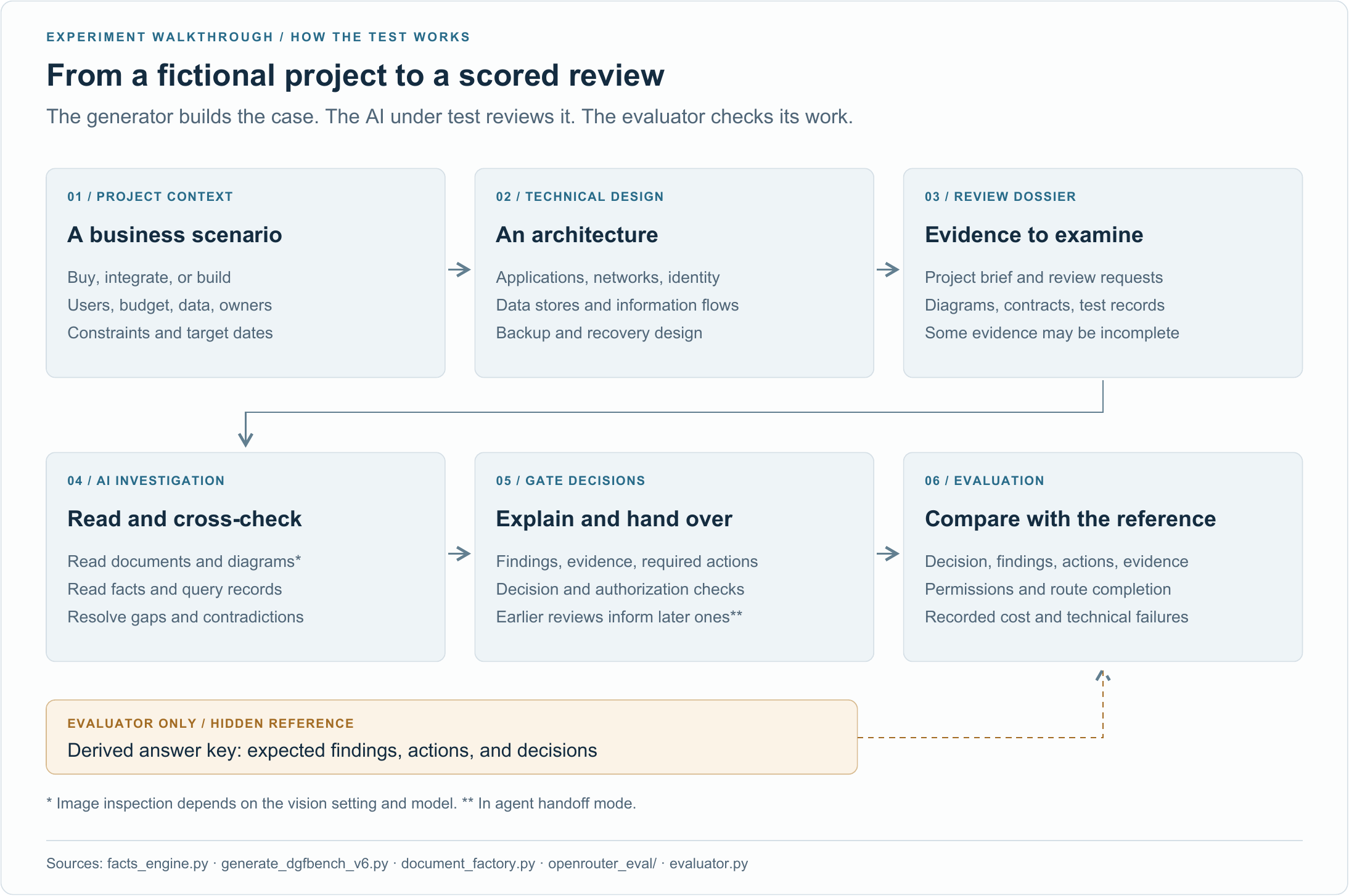}
\caption{Implemented experimental workflow. A program generates the case and evidence; the tested
model investigates and submits reviews; a deterministic evaluator scores the retained trace.
Public authoritative facts are available in addition to documents.}
\label{fig:experiment-workflow}
\end{figure}

\subsection{Execution and scoring protocol}

Three exact OpenRouter endpoint identifiers were evaluated on 22--23 September 2026:
\nolinkurl{google/gemini-3.8-flash}, \nolinkurl{openai/gpt-5.6-luna}, and
\nolinkurl{deepseek/deepseek-v4.1-flash}. Temperature was zero, vision was \texttt{auto},
and each gate permitted at most 20 turns, 40 tool calls, and 8,192 output tokens per turn.
Of 900 planned model--project runs, 899 and their 5,094 gates are evaluable. One Gemini
Integrate run failed at its first gate because of a provider generation error; its six gates
are excluded. Recorded failed attempts remain in the cost accounting and released traces.

The models are deployed endpoints, without training or fine-tuning in this study. Gemini was
recorded through Google AI Studio and Luna through OpenAI; DeepSeek used multiple OpenRouter
providers. The identifiers are provider-reported, not independent audits of model weights.
Gates run sequentially within each dossier, while dossiers can run concurrently. Interrupted
collection resumes completed work and valid gate checkpoints rather than selecting the best
answer from competing trajectories. No final model-identifier mismatch or agent-protocol failure
is recorded among the evaluable original runs. Recorded truncated responses number 14 for
Gemini, zero for Luna, and two for DeepSeek; they are not equivalent to incomplete dossiers.

The frozen deterministic evaluator requires a correct disposition, exact finding and action
sets, observed evidence references and exact observed field/value excerpts, and the correct authorization
indicator. A \emph{strict gate success} passes every component; a \emph{complete route} passes
every gate. Disposition accuracy alone therefore does not measure a complete review. Validated
conditional approvals can change the effective reference outcome. A false approval proceeds
when that effective reference requires rework, suspension, or refusal; a critical miss omits
a finding classified as critical by the policy.

For a model, strict gate success divides successful gates by all gates in its evaluable dossiers;
complete-route success divides dossiers whose every gate succeeds by the evaluable dossier count.
A six-gate route therefore contributes six observations to the gate rate, while a five-gate route
contributes five. Passing four of five gates yields 80\% gate success within that case and zero
complete-route success. Finding and action sets must each attain F1 equal to one; averaging good
components does not compensate for a missing required component. The frozen protocol is
\texttt{DGF-decision-v8.2-authorized-review}. Authorization checks replay validated mandates and
actions; General uses effective upstream opinions after permitted conditional approvals.

This separates the quality of the final record from the behavior leading to it. A tool may
reject an impermissible commitment, after which the model submits a valid refusal. Conversely,
a model can write an appropriate explanation but submit the wrong disposition field. Both the
structured outcome and preceding tool events are retained, allowing these mechanisms to be
examined instead of inferring them from a single aggregate score.

\subsection{Deterministic control and aggregate results}

A deterministic control reads the same public snapshots and executes the public policy functions,
then writes its predictions before a separate process invokes the frozen scorer. It does not
read hidden references during generation, request optional conditional approvals, or extract
facts from documents. It can quote full snapshots without the models' per-turn output-token
limit. Its perfect score establishes that the formalized review kernel is software-executable.
The agent comparison therefore measures execution of that supplied kernel and its evidence
contract. An incremental contribution to fact recovery requires the matched-information
comparison specified in Section~\ref{sec:implications}.

The control constructs findings and aggregates dispositions from the public interface. It
propagates its own preceding decisions to General. A read guard blocks evaluator-only files,
and generation asserts that the scorer was not imported; a second process then scores the
saved outputs. Retaining the base disposition without requesting optional risk acceptance is
permitted, so conforming systems need not submit identical labels when one uses a valid mandate.
The control's 1,700 strict successes and 300 complete routes demonstrate sufficiency of the
prepared policy/fact interface. Its zero model API expenditure excludes programming and local
computation, and it supplies no evidence about the cost of producing those inputs.

\begin{table}[!htb]
\centering\small
\caption{Task execution on the original population and repeated trajectories. Counts preserve
the original strict endpoint. Sensitivity accepts specified structural evidence matches;
it is post-hoc and is not a semantic adjudication. The deterministic control was not repeated.}
\label{tab:compact-evidence}
\setlength{\tabcolsep}{4pt}
\begin{tabularx}{\linewidth}{@{}Lrrrr@{}}
\toprule
Outcome & Gemini & Luna & DeepSeek & Rules\\
\midrule
Evaluable projects & 299 & 300 & 300 & 300\\
Evaluable gates & 1,694 & 1,700 & 1,700 & 1,700\\
Correct dispositions & 1,694 & 1,668 & 1,619 & 1,700\\
Strict gate successes & 1,609 & 1,416 & 1,261 & 1,700\\
Strict gate success (\%) & 94.98 & 83.29 & 74.18 & 100.00\\
Complete routes & 230 & 127 & 74 & 300\\
Complete routes (\%) & 76.92 & 42.33 & 24.67 & 100.00\\
Evidence-only failures & 85 & 221 & 338 & 0\\
False approvals / critical misses & 0 / 0 & 1 / 3 & 1 / 1 & 0 / 0\\
Sensitivity: successful gates & 1,678 & 1,454 & 1,313 & ---\\
Sensitivity: complete routes & 284 & 144 & 85 & ---\\
\midrule
\multicolumn{5}{@{}l}{\emph{Follow-up: 15 existing dossiers, three trajectories per model}}\\
Strict gates / 255 & 245 & 211 & 187 & ---\\
Complete routes / 45 & 35 & 19 & 11 & ---\\
Dossiers passing all three / 15 & 9 & 3 & 0 & ---\\
\bottomrule
\end{tabularx}
\end{table}

\paragraph{Overall performance and uncertainty.}
Gemini reaches the correct effective disposition on every evaluable gate and satisfies all
non-evidence components; its 85 strict failures concern evidence conformity. Its strict gate
rate is 94.98\% (95\% interval 93.86--96.04), but only 76.92\% of complete routes pass
(71.82--81.34). Luna and DeepSeek also exhibit a large gate--route gap, alongside substantive
errors: Luna approves one Legal review despite a missing data-processing agreement, and
DeepSeek submits a placeholder \texttt{GO} at forced finalization when address overlap
requires \texttt{NO\_GO}. The model comparison therefore distinguishes successful disposition,
complete documented review, and dependable execution of a whole route.

Gate intervals resample whole cases within routes (2,000 draws, seed 81931); route intervals
use Wilson bounds. They describe the generated population, retaining within-project dependence.
They do not treat 5,094 gates as independent organizations. The common 299-case comparison
preserves the ranking; its paired intervals and complete component tables are released.

Luna's strict gate interval is 81.29--85.24\%, and DeepSeek's is 71.88--76.47\%; their route
intervals are 36.87--47.99\% and 20.13--29.84\%. Pairwise comparisons resample the same 299
common dossiers together within routes using 10,000 draws and NumPy seed 81931. Gemini's gate
advantage is 20.90 percentage points over DeepSeek (95\% interval 18.30--23.44) and 11.75 over
Luna (9.62--13.93); Luna exceeds DeepSeek by 9.15 points (6.67--11.63). These comparisons
control the case mix shared by the models. They do not establish a stable ranking over new
organizations, later endpoint versions, or different information conditions.

\paragraph{Where performance changes.}
Gate families expose different demands (Table~\ref{tab:experiment-gates}). Gemini passes all
IT and Compliance reviews strictly but 85.43\% of Architecture and 87.96\% of General reviews.
Luna exceeds Gemini in Architecture and Tech Readiness while falling behind at Procurement and
General. Thus the overall ordering does not imply dominance on every component of the workflow.
DeepSeek's Procurement decision accuracy is 97\%, while its strict Procurement success is 43\%;
Luna's corresponding rates are 96\% and 33\%. Evidence handling, rather than disposition alone,
accounts for much of this contrast.

\begin{table}[!htb]
\centering\small
\caption{Strict gate success by review family (percent). Columns retain the original generated
table order. Gemini has one fewer Integrate dossier, affecting six of its gate families.}
\label{tab:experiment-gates}
\begin{tabular}{@{}lrrr@{}}
\toprule
Gate & DeepSeek & Gemini & Luna\\
\midrule
\input{generated/table_benchmark_gates.tex} 
\bottomrule
\end{tabular}
\end{table}

At route level, Gemini completes 87 of 100 Buy, 75 of 99 Integrate, and 68 of 100 Build cases.
Luna completes 20, 52, and 55; DeepSeek completes 13, 32, and 29, each out of 100. Route length
alone cannot explain these differences, because route composition and project contents also
change. Sequential handoffs make an upstream error available to General, but this single
handoff condition cannot isolate a causal propagation effect. Luna's missed upstream
\texttt{NO\_GO} in General illustrates why a downstream omission may repeat an earlier problem
rather than constitute an independent safety event.

\input{sections/02_authorized_decisions}

\subsection{Evidence failures and two inspectable traces}

An audit of all 690 gates with a failed evidence component leaves the primary scores unchanged
and applies one declared sensitivity rule: accept exact relevant field/value subsets from one
observed object despite field order or equivalent numeric serialization. No quote is repaired,
no tool observation invented, and no join across different objects promoted. This recovers
69 Gemini, 38 Luna, and 52 DeepSeek gates (Table~\ref{tab:compact-evidence}). Gemini's rate
becomes 99.06\% at gate level and 94.98\% at route level. The interpretation is specific:
much of this model's measured route failure depends on the evidence representation contract.
The audit does not establish full premise coverage or independently validated semantic support.

The distinction between provenance and support also appears in Procurement. Across 127
evidence-failed Procurement reviews, all 205 submitted support entries cite sources actually
read and included in the submission. DeepSeek provides 26 failed support entries citing CSV
sources; 25 are literal rows in observed tool output. The snapshot-only source-ID requirement
rejects those entries before checking their contents. Other authentic excerpts omit a decisive
field, and some combine values from different objects. An observed quotation can therefore
be contract-excluded, relevant but incomplete, or structurally nonliteral. None of these labels
alone is a semantic truth judgment. The released audit preserves these distinctions without
adding another recovered-success rate.

\paragraph{Trace 1: a delegated readiness review.}
In \nolinkurl{DGF-BLD-035200_build}, a generated vendor-portal project, Gemini's Tech Readiness
review finds a draft runbook. It names \texttt{TR-OPS-001}, proposes
\texttt{COMPLETE\_OPERATIONAL\_HANDOVER}, and cites the observed value
\texttt{"runbook\_status": "draft"}. After consulting an active mandate, it invokes the
conditional-approval tool with a requirement to approve the runbook before production cutover.
The tool executes the authorization but explicitly leaves the finding unresolved. The model
submits \texttt{GO\_WITH\_RESERVATIONS} with the finding and action retained, satisfying the
strict contract. This post-hoc example demonstrates a performed review and authorized workflow
transition, with the runbook condition still open. Table~\ref{tab:approval-behavior} places
this event within the complete set of recorded approval requests and final decisions.

The dossier is Project Falcon, a fictional HR vendor portal for 1,000 users, with a requested
budget of EUR~500,000 and confidential data. Its architecture artifact is reproduced in
Figure~\ref{fig:experiment-architecture}. All three models strictly pass its five gates, with
dispositions \texttt{GO}, \texttt{GO}, \texttt{GO}, \texttt{GO\_WITH\_RESERVATIONS}, and
\texttt{GO} in route order. This selected case shows a complete successful route; it was not
selected to estimate a population rate. Image availability does not show that image interpretation
was necessary, since the models could consult structured facts.

\begin{figure}[!htb]
\centering
\includegraphics[width=\linewidth]{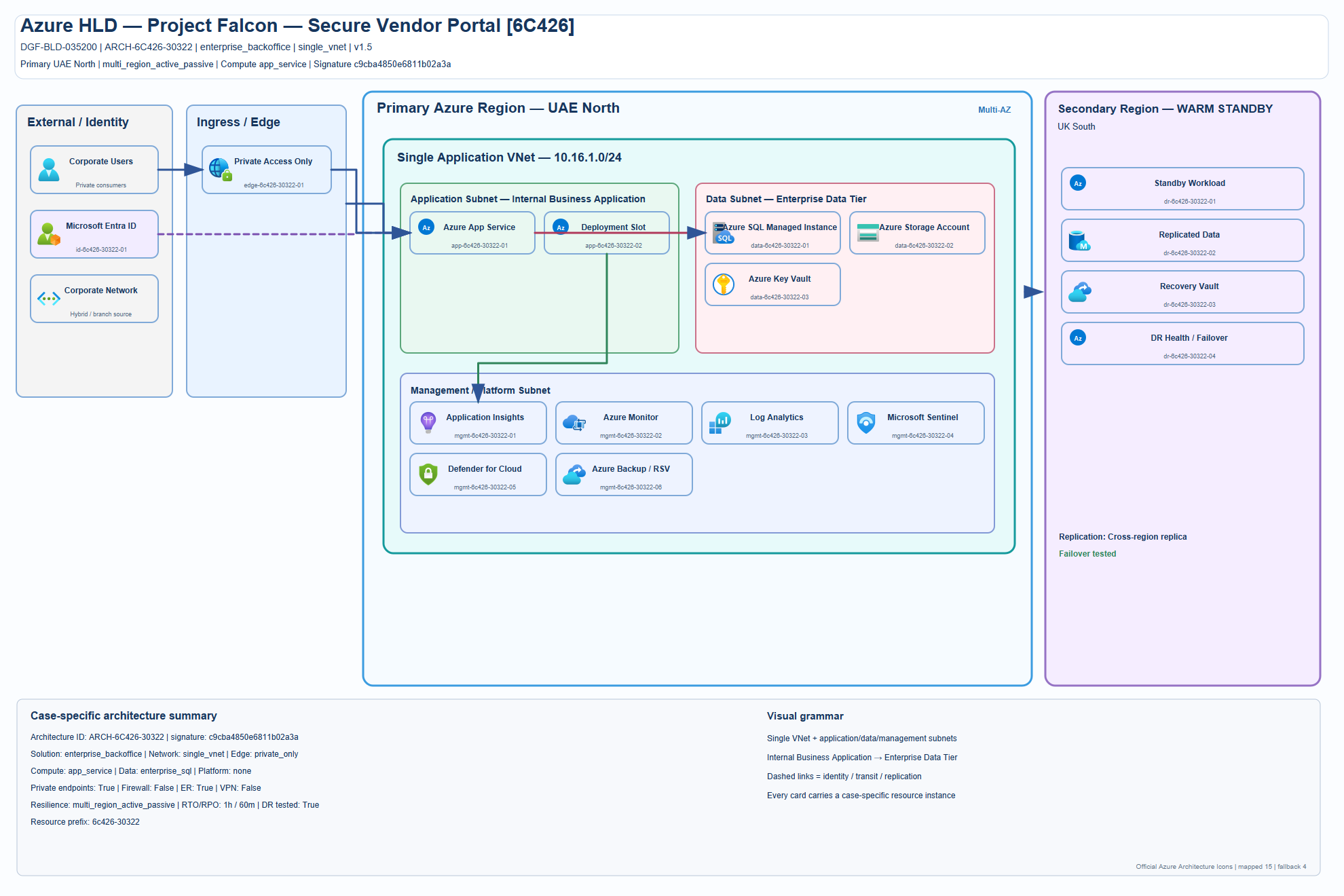}
\caption{The generated architecture document in Project Falcon's actual dossier. The diagram
is review evidence, not an independently validated production design. Identity, application,
data, monitoring, backup, and recovery elements contextualize the readiness example.}
\label{fig:experiment-architecture}
\end{figure}

\paragraph{Trace 2: correct architecture judgment, failed literal proof.}
In \nolinkurl{DGF-BLD-035201_build}, Project Meridian, Gemini reads that an API gateway is
required but absent, and that measured latency is 55~ms against a 20~ms target. It correctly
identifies \texttt{ARCH-API-001} and \texttt{ARCH-PERF-001}, proposes adding the gateway and
remediating latency, and submits \texttt{REWORK}. Two attempted conditional approvals fail:
one omits an open finding and the other includes a finding ineligible for risk acceptance.
The final submission complies with the required refusal to proceed.

Its evidence excerpt puts \nolinkurl{api_gateway_present} before
\nolinkurl{api_gateway_required}; the observed record contains the reverse order. Both values
are correct, but their concatenation is not an exact observed sequence. Evidence fidelity is
0.5 across the two findings, while disposition, findings, actions, and authorization all score
one. Strict success is consequently zero, despite a weighted component score of 0.95. This
trace illustrates why a complete-review metric must make its representation requirements
explicit. Both examples were chosen after evaluation for explanation; their original paths,
outputs, and tool events are released, and unsuccessful cases remain in every aggregate.

\subsection{Repeated trajectories, expenditure, and interpretation}
The follow-up fixes 15 existing dossiers, five per route, before collecting three new
trajectories for every model--dossier pair: 135 completed runs and 765 gates. No best-of-three
selection is used. Gemini, Luna, and DeepSeek pass respectively 35, 19, and 11 of 45 routes;
only nine, three, and zero dossiers pass on all three trajectories. All follow-up false-approval
and critical-miss counts are zero. Provider routing is not fixed, so these observations measure
repeatability of the recorded system rather than intrinsic model randomness. Original and
follow-up recorded API costs are USD~87.015829172 and USD~12.5598870228, totaling USD~99.58;
development, local computation, human review, and deployment are not costed.

The selection uses sorted dossier identifiers and a fixed random seed, 23092026, before the
new outcomes are collected. Each repetition contains 85 gates per model. Gemini's strict
success counts are 82, 82, and 81; Luna's are 73, 71, and 67; DeepSeek's are 63, 64, and 60.
The corresponding complete-route counts are 12, 12, and 11; 8, 7, and 4; and 3, 4, and 4.
Pooling the trajectories retains all three runs rather than selecting a successful attempt.
Complete-route pass/fail varies across repetitions on five Gemini, six Luna, and eight DeepSeek
dossiers. Zero observed all-three success for DeepSeek is not proof of zero population
probability; the follow-up contains only 15 distinct synthetic projects.

The follow-up retains the original endpoints, temperature, vision setting, agent handoffs,
and gate resource limits, with three workers and one concurrent dossier per model. Its archive
contains 48 error checkpoints: 46 API-key-limit refusals and two provider-finish errors. These
are not counted as additional completed runs. Two Gemini trajectories resume seven previously
completed gates, retained once in the final outcomes. Resumption preserves the planned sample;
it does not turn failed attempts into independent observations.

Exploratory intervals resample five whole dossiers within each route, retaining all three
trajectories and all gates together (10,000 draws, seed 24092026). Pooled strict gate success
is 96.08\% for Gemini (95\% interval 93.33--98.43), 82.75\% for Luna (75.29--89.80), and
73.33\% for DeepSeek (66.67--79.61). These rates resemble the original ordering, while the
all-three endpoint exposes dossier-level repeatability. Under the same structural sensitivity
rule, repeated gate counts rise to 249, 217, and 195 out of 255; only four of Gemini's ten
evidence failures are recovered, compared with 69 of 85 in the original study. The original
recovery proportion is therefore not assumed to generalize unchanged.

Original recorded expenditure is USD~71.88 for Gemini, USD~5.05 for Luna, and USD~10.09 for
DeepSeek; the follow-up adds approximately USD~10.39, USD~0.75, and USD~1.42, respectively.
Costs include recorded failed attempts and depend on provider routing, caching, and retry
history. They are reproducible accounting observations for these runs, not permanent prices
or estimates of cost per employee replaced. The repeated-run archive retains key-limit and
provider failures, including resumed prefixes counted once in their completed trajectory.

The information counterexample in Section~\ref{sec:contract} identifies a different failure
boundary: some proposed document-only inputs cannot determine the required decision. Together,
the observed executions and the source check show both an automatable formalized task and
a specific obstruction that must be resolved before extending its scope.
The experiments also separate three engineering objectives: producing the required decision,
preserving an accepted evidence record, and delivering the entire route repeatedly. Improving
one does not automatically improve the others. This decomposition gives an FDE implementation
concrete targets for policy execution, evidence services, and exception handling, while the
labor account determines whether their combined operation substitutes for human work.

%% file: generated/table_benchmark_gates.tex
Architecture & 75.00 & 85.43 & 89.50\\
Compliance & 88.00 & 100.00 & 96.00\\
General & 68.67 & 87.96 & 61.67\\
IT & 87.00 & 100.00 & 97.67\\
Legal & 68.50 & 99.50 & 80.50\\
Procurement & 43.00 & 96.00 & 33.00\\
Security & 72.33 & 98.66 & 92.00\\
Tech Readiness & 71.00 & 89.00 & 97.00\\

%% file: sections/02_authorized_decisions.tex
\subsection{Authorized conditional decisions}

The review contract permits an agent to consult a standing mandate and request conditional
approval for eligible findings. Eligibility requires open findings, permission to accept each
finding under the public policy, and an active mandate for the particular gate and phase.
The environment checks the request and records its effective authorization state. The final
review retains the finding and the required action even when conditional approval is used.
This operation represents bounded delegated decision-making within the synthetic workflow.

Table~\ref{tab:approval-behavior} summarizes the released audit of the original retained
checkpoints. Calls and gates are different units: one gate can contain several requests.
An approval is counted as used when a replay-validated final approval accompanies the model's
\texttt{GO\_WITH\_RESERVATIONS} submission. The all-gate eligibility denominator can differ
between models because authorized upstream decisions affect General's consolidation.

\begin{table}[!htb]
\centering\small
\caption{Conditional-approval behavior in the original runs. The final row restricts the
comparison to the same 391 eligible non-General gates in the 299 shared dossiers. Calls
rejected by the environment are recorded separately from final review outcomes.}
\label{tab:approval-behavior}
\begin{tabular}{@{}lrrr@{}}
\toprule
Recorded quantity & Gemini & Luna & DeepSeek\\
\midrule
Approval calls & 864 & 394 & 216\\
Rejected calls & 466 & 40 & 14\\
Used conditional decisions & 398 & 353 & 202\\
Eligible gates, all scored runs & 398 & 398 & 397\\
Used on matched opportunities / 391 & 391 & 347 & 201\\
\bottomrule
\end{tabular}
\end{table}

On the matched opportunities, the corresponding use rates are 100.00\%, 88.75\%, and 51.41\%.
The public contract permits retaining the base decision as well as requesting an eligible
approval; use frequency is therefore a behavioral measurement. Some approvals preserve a
disposition that was already conditional. Subtracting base-decision agreement from effective
agreement would consequently measure a different quantity. Luna also has one validated
approval followed by a final \texttt{REWORK}, which explains why its 354 validated approval
states yield 353 used conditional decisions.

The recorded outcome belongs to the agent operating with its authority-enforcement tool.
A rejected request is contained by that tool; a successful request remains subject to its
recorded conditions. The synthetic contract checks that condition strings are nonempty;
their operational adequacy has no separate adjudication in this experiment. These explicit
semantics make it possible to examine decision quality, delegated action, and retained
obligations separately in a deployment. The Falcon example below shows one such conditional
decision together with its supporting evidence and open action.

%% file: sections/06_validation.tex
\section{From task substitution to a workforce hypothesis}
\label{sec:implications}

The observations support a positive, bounded conclusion: software can execute the specified
governance-review operations in this engineered environment. The deterministic control makes
the case for an executable decision kernel especially clear. The agents demonstrate another
implementation of many of these operations, with measurable differences in evidence handling,
reliability, and cost. These components can be combined in the proposed deployment architecture.
Its operating evaluation measures the effort required to construct and maintain the interfaces
when facts and policies arrive in less prepared forms.

This is also where the claim about human replacement becomes testable. If a deployed system
performs accepted tasks previously executed by people and satisfies the labor inequality at
unchanged output and quality, it has substituted for part of their work. Calling the remaining
people supervisors does not restore the eliminated execution hours. At fixed demand, fewer
required hours can support fewer posts; actual headcount also depends on allocation, utilization,
redeployment, and new work. The argument therefore permits substantial workforce displacement
without treating it as a measured consequence of this synthetic experiment.

\subsection{The 2033 hypothesis and a falsifiable deployment test}

The author's original hypothesis is that the DGF ecosystem will require at least 80\% fewer
workload-equivalent FTE by 20 September 2033 than during the reference year ending on
20 September 2026, at comparable governed output, quality, and service. The boundary includes
preparation, exceptions, verification, remediation of review errors, maintenance, supervision,
and supplier support. On an illustrative 140-FTE baseline, the target is at most 28 FTE.
This is an author-selected forecast, not an extrapolation fitted to benchmark scores.

The baseline window has already closed and no enterprise cohort has been enrolled. Testing
this exact dated hypothesis requires auditable historical activity records and a declared
population and aggregation rule before observing its endpoint. Otherwise it cannot be scored
as confirmed. A newly registered baseline can support a separate seven-year prospective test,
but cannot move the date of the original prediction. A credible test fixes the accounting
boundary, useful-hours conversion, output mix, quality tolerances, and service levels. It
reports the complete $L_A/L_H$ account rather than headcount in a renamed department. A ratio
above 0.20 at the deadline fails the proposed 80\% target; achieving that ratio while degrading
the fixed quality or service requirements also fails the substitution claim. Missing records
make the result unassessable, not successful.

Full elimination of human DGF execution is a stronger, undated conjecture. It would require
both case-level human effort and the recurring support floor to disappear. No finite benchmark
here establishes that endpoint. The scientifically actionable question is how much accepted
task execution can already be transferred and what prevents further transfer.

\subsection{Observing substitution without losing the work boundary}

A deployment study needs to observe both accepted output and the effort that produces it.
Begin with a defined population of review obligations, not a list of job titles. A project
may trigger several gates, return for revision, or draw on a shared specialist. Activity
records should preserve those relationships so that the same hour is not counted at every
gate and a returned dossier does not disappear from the denominator. The endpoint is the
delivery of the governed work, including obligations left open by conditional decisions.

The baseline can combine case records, activity logs, timesheets, and supplier records under
a common coding scheme. Evidence preparation, ordinary review, exception work, correction,
handoffs, and support need distinct codes. Shared infrastructure work is allocated once using
a declared rule, with an unallocated category when the attribution is unknown. Unobserved
supplier hours are missing observations; a low internal headcount cannot stand in for them.
Initial integration is reported separately from recurring operations, then included in any
investment or payback calculation. A deployment with low steady-state effort can still have
a substantial transition cost.

There are two different comparisons. Actual total hours answer how much work the organization
performed during the period. Standardized hours answer how much work comparable output would
require. If case volume falls, simply comparing payroll totals can mistake reduced demand for
automation. Conversely, rising project volume can conceal a reduction in effort per comparable
case. A study should report both quantities, standardize the variable work for route and case
mix, and retain a separately measured support term. Multiplying a fixed support burden by a
case-volume adjustment would erase the very floor the accounting model is intended to expose.

Quality and service are assessed alongside hours. Appropriate measures depend on the declared
contract, but include missed mandatory findings, invalid commitments, uncompleted obligations,
correction after acceptance, and time to a usable decision. The deployment must predeclare
which errors invalidate an output and which service changes are tolerable. Faster approvals
are not an efficiency improvement if blocking defects go undetected. Equally, a system that
escalates every dossier may preserve final decision quality while adding an extra processing
layer; the labor account exposes that outcome rather than rewarding abstention alone.

For an enrolled population, uncertainty in the final workload ratio should include between-case
variation, organizational clustering, and uncertainty in reconstructed baseline hours. An
upper bound at or below 0.20, with the quality and service conditions met, supports the target
for that population. A lower bound above 0.20 contradicts it; an interval crossing the threshold
is inconclusive. Missing material work cannot be silently assigned zero. Attrition, failed
deployments, and changes in the suppliers included in the account need explicit reporting.

Observing a reduction and attributing it to agents are separate steps. Staged or randomized
adoption, when feasible, supplies a comparison for changes that would have occurred anyway.
Otherwise, a study must document contemporaneous process redesign, outsourcing, demand shifts,
and other automation. An organization can meet the workload threshold while the causal share
attributable to language models remains unresolved. Finally, a convenience sample of successful
implementations does not establish an ecosystem-wide forecast. Such inference requires a
declared sampling frame, weights, and coverage of organizations where integration fails or
does not occur. No enterprise observations of this kind are claimed in the present paper.

\subsection{The next discriminating experiment}

The most useful extension varies information preparation while preserving answerability.
Compare agents receiving structured snapshots with agents receiving the source documents and
operational records from which the same facts can be recovered. Separately vary executable
versus natural-language policy. Fix source access, mandate, time and tool budgets, and the
evidence contract before testing on held-out cases. Include the rule engine and any required
extraction pipeline in the comparison. Report decision quality, supported findings, complete
routes, abstention, and all engineering and operating costs.

\begin{table}[!htb]
\centering\small
\caption{Proposed matched-information design. These conditions describe an unexecuted
extension, not four completed experiment arms.}
\label{tab:proposed-ablation}
\begin{tabularx}{\linewidth}{@{}P{1.6cm}LL@{}}
\toprule
Condition & Policy representation & Permitted evidence\\
\midrule
A & Executable definitions & Authoritative structured snapshots\\
B & Semantically equivalent prose & The same snapshots\\
C & The same executable definitions & Documents and operational records sufficient to recover the same facts\\
D & The same prose as B & The same documents and records as C\\
\bottomrule
\end{tabularx}
\end{table}

Table~\ref{tab:proposed-ablation} separates policy application from fact recovery. A versus B
changes policy representation while holding information fixed. A versus C changes information
presentation while retaining executable policy. The fourth condition checks the combined
burden. Access to code execution must be held constant; otherwise a representation comparison
also changes tool capability. Prose policies must preserve thresholds, priority rules, and
exceptions, and finding catalogs must not reintroduce the executable predicates in an arm
intended to omit them. These are controls on the experiment, not assumptions that the two
representations have already been made equivalent.

This extension does not require claiming human-level accuracy. It asks which parts of a review
the system can perform when the required facts are no longer supplied as a ready-to-use snapshot.
Our counterexample means that simply deleting snapshots is not a valid design. The released
source audit also identifies runtime tools that can return snapshots, so changing file visibility
alone would not isolate the intended condition. The existing extraction ablation and General-only
intervention are preparations, not additional measured results in this paper.

The evidence endpoint must also be common across the new conditions. The frozen benchmark
requires some support to name a structured snapshot. Removing that source while retaining the
same identifier requirement would penalize a source-reading agent even when it found the right
fact. A new contract can instead require an observed source identifier, a field or document
span, and a faithfully copied value. It must separately check whether those values support the
finding's premises. Applying that contract prospectively creates a new endpoint; it does not
retroactively change the original strict scores.

For the handoff question, replay the General gate under saved model history and a reference
history while preserving the same validated upstream conditional approvals. Replacing a
lawful conditional decision with an unconditional blocking reference would change authority
as well as information quality. The intervention should isolate the accuracy and completeness
of the handoff, retain identical dossier access and budgets, and use paired repeated
trajectories. Such a design could distinguish propagation of an earlier error from a fresh
General-gate error; the observed whole-route scores alone cannot identify that causal effect.

\subsection{Scope of the observations}

The policy, generator, and evaluator use the same declared rule system, so the measured endpoint
is conformity to that system. The 300 synthetic cases are sampled for decision coverage.
The original collection contains one trajectory per model and case across three endpoints;
the follow-up repeats 15 of those dossiers. The evidence sensitivity analysis applies a
specified structural matching criterion. These choices define the population, information
condition, and acceptance rule attached to the reported scores.

The observed comparison concerns agent and deterministic implementations of review contracts.
Human-review performance, enterprise operating effort, staffing changes, and the relative
adoption speed of different business functions are outside the measured variables. The
deployment protocol specifies how to observe the labor account and evaluate the dated
workforce hypothesis. Review performance and workforce requirements thus have explicit,
different measurement procedures.

%% file: sections/07_conclusion.tex
\section{Conclusion}

Agents can replace human execution of governance-review tasks when they produce the required
decision, findings, evidence, actions, and authorization record under the accepted contract.
DGF-Bench provides a concrete demonstration of this capability in a controlled environment:
the tested agents complete specialist reviews and entire project routes, and the deterministic
control executes the formalized policy kernel. Governance obligations can therefore remain in
force while their execution transfers from people to software.

The DGF is a credible candidate for extending that transfer. Its recurring dossiers, named
review responsibilities, explicit policies, and recorded handoffs give forward deployed
engineers identifiable work to implement. The proposed architecture connects agent
investigation to evidence services, policy execution, authority enforcement, and exception
handling. As these components cover more accepted tasks, the share requiring case-by-case
human execution can decrease across the workflow. The same reasoning applies to preparation,
verification, and support when their own contracts can be fulfilled in software.

The consequence for work is direct. At comparable output, quality, and service, a system that
satisfies the substitution frontier requires fewer human hours to perform the governed work.
Under the staffing conditions stated in this paper, sufficiently large reductions remove the
need for some posts. Organizations may absorb the released capacity through new demand or
redeployment; the execution of the transferred tasks has still been replaced. The relevant
quantity is the complete human-work account, including the people who operate, maintain, and
resolve exceptions around the system.

The long-term implication is that agents, integrated with executable policies and delegated
authority, can progressively replace human execution across DGF workflows. The measurements
here concern review performance; the magnitude and timing of workforce displacement are
specified as deployment hypotheses. The author's target of 80\% fewer required DGF FTE by
2033 is evaluated against the fixed baseline, output, quality, and accounting boundary
defined in this paper. Complete replacement would additionally require the remaining
exception and support work to be executed in software.

The contribution is a testable path from successful agent reviews to human task substitution:
define the contract, establish dependable execution, measure the work that remains, and expand
the scope where accepted output requires less human effort. Human review is an implementation
of governance that can change as these conditions are met.

\section*{Data, code, and attribution}

Jeremy Canale is the author: \texttt{contact@jeremycanale.com},
\url{https://www.jeremycanale.com}, \url{https://www.linkedin.com/in/jcanale13}.
The public repository is \url{https://github.com/jeremy1392/dgf-agentic-bench}.
Its \href{https://github.com/jeremy1392/dgf-agentic-bench/releases/tag/dgf-bench-300-20260923}{\texttt{dgf-bench-300-20260923}}
release contains the source snapshot, complete generated dossiers including Word documents and
architectures, model traces, failed attempts, scores, and cost ledgers. Results and reproduction
are in \nolinkurl{research/2026-09-dgf-bench/}; controls, evidence audits, repetitions, and source
checks are in \nolinkurl{research/2026-09-followup/}. Reproduce the original runs with the frozen
source. Workforce parameters are in \nolinkurl{paper/anc/parameters.json}; this manuscript
and its verifier are in \nolinkurl{paper2/}, and the extended manuscript in \nolinkurl{paper/}.
Earlier versions remain in repository history. The author is responsible for the claims and
interpretation.

%% file: sections/08_reproduction.tex
\appendix
\section{Reproducing the claims from the released artifacts}
\label{sec:reproduction}

The public release separates the experimental record from the current development checkout.
Use the frozen benchmark source for the original run; later code changes need not implement
the same scoring protocol. The dataset archive contains every generated dossier, including
Word documents, architecture diagrams, policies, mandates, and evaluator references. The run
archive contains model outputs and tool events as well as scores, failures, and cost records.
Per-file inventories and archive SHA-256 values identify the inputs before recomputation.

\begin{table}[!htb]
\centering\small
\caption{Repository records for the principal claims. Paths are relative to the public
repository; the release linked in the data statement contains the larger underlying archives.}
\label{tab:reproduction-map}
\begin{tabularx}{\linewidth}{@{}P{3.1cm}L@{}}
\toprule
Claim or operation & Inspectable record\\
\midrule
Original metrics and intervals & \nolinkurl{research/2026-09-dgf-bench/paper_results.json}\\
Matched model comparisons & \nolinkurl{research/2026-09-dgf-bench/paired_comparisons.json}\\
Falcon and Meridian traces & \nolinkurl{research/2026-09-dgf-bench/worked_examples.json}\\
Structural evidence audit & \nolinkurl{research/2026-09-followup/ALL_MODELS_AUDIT.md}\\
Repeated trajectories & \nolinkurl{research/2026-09-followup/repetition_analysis.json}\\
Information counterexample & \nolinkurl{research/2026-09-followup/document_ablation_preflight.json}\\
Workforce calculations & \nolinkurl{paper/anc/parameters.json}\\
Printed-table consistency & \nolinkurl{paper2/verify_manuscript.py}\\
\bottomrule
\end{tabularx}
\end{table}

The original offline reproduction script verifies downloads against the inventories, checks
the source and dataset identities, and recalculates aggregates and paired comparisons. The
expected denominators are 899 evaluable trajectories and 5,094 gates, with 299 dossiers in
the common comparison. Within the run archive, use \texttt{analysis\_20260923\_final}; earlier
partial exports are retained as historical records. No new API calls are needed to reproduce
these summaries. Recomputing saved-score aggregates verifies arithmetic and exclusions;
rescoring saved submissions with the frozen evaluator is a separate check of implementation.
Neither operation independently establishes the professional validity of the policy.

For the repeated runs, keep every planned trajectory and its retained gate prefix once.
Infrastructure errors remain in the archive and cost ledger but do not become extra completed
trials. Reproduce the pooled rates by summing successful gates over 255 gates per model, and
complete routes over 45. For the all-three endpoint, first group trajectories by dossier and
require success in every repetition, then divide by 15. Resampling must retain this grouping;
treating the 45 trajectories as 45 distinct projects would change the uncertainty calculation.

For the labor figure, the parameter file supplies population FTE, visit counts, exception
fractions, conditional effort assumptions, and operating scenarios. Evaluate the dimensional
account separately for each population, sum hours, and divide by the declared 120-hour monthly
FTE convention. This recovers 43.52, 85.32, 89.07, and 168.92 FTE. The computation validates
the consequences of the specified assumptions; it does not estimate those assumptions from
the benchmark. The manuscript verifier checks the printed numerical records and citations,
while source-archive compilation checks that the document is self-contained.

%% file: sections/references.tex
% Checked primary-source metadata, 24 September 2026. Cite only entries used.